\pdfoutput=1
\documentclass[5p,times]{elsarticle}
\usepackage[T1]{fontenc}
\usepackage{amsmath,amssymb,amsthm}
\usepackage{array,booktabs,tabularx}
\usepackage{graphicx}
\usepackage{xcolor}
\usepackage{microtype}
\usepackage[hidelinks]{hyperref}
\usepackage{xurl}
\newtheorem{proposition}{Proposition}
\newtheorem{definition}{Definition}
\theoremstyle{remark}
\newtheorem{remark}{Remark}
\newcommand{\auditBenign}{22}
\newcommand{\auditBenignPct}{55}
\newcommand{\auditError}{4}
\newcommand{\auditErrorPct}{10}
\newcommand{\auditMaterial}{14}
\newcommand{\auditMaterialHi}{50}
\newcommand{\auditMaterialLo}{22}
\newcommand{\auditMaterialPct}{35}
\newcommand{\auditN}{40}

\newcommand{\auditReal}{36}
\newcommand{\auditRealHi}{96}
\newcommand{\auditRealLo}{77}
\newcommand{\auditRealPct}{90}
\newcommand{\avoidCode}{81.6}
\newcommand{\avoidDecl}{83.4}

\newcommand{\avoidFail}{45.3}
\newcommand{\causeCodePct}{88}
\newcommand{\ccaoAfter}{141}
\newcommand{\ccaoBefore}{264}
\newcommand{\ccaoEval}{50}
\newcommand{\ccaoStale}{83}
\newcommand{\ccaoStalePct}{31}
\newcommand{\chCodeStale}{8.9}
\newcommand{\chConfigStale}{5.7}
\newcommand{\chDataStale}{0.9}

\newcommand{\cmEnv}{12}

\newcommand{\costCode}{1.8}

\newcommand{\evalUnsafeDecl}{0.9}
\newcommand{\execCode}{18.4}
\newcommand{\execCodeHi}{27.0}
\newcommand{\execCodeLo}{12.8}
\newcommand{\execDecl}{16.6}
\newcommand{\execDeclHi}{25.0}
\newcommand{\execDeclLo}{11.4}
\newcommand{\execEnv}{26.2}
\newcommand{\execEnvHi}{35.3}
\newcommand{\execEnvLo}{19.8}
\newcommand{\execFail}{54.7}
\newcommand{\execFailHi}{62.2}
\newcommand{\execFailLo}{48.6}
\newcommand{\execRefCode}{18.7}
\newcommand{\gateNineP}{0.40}
\newcommand{\gatePerRecord}{42}
\newcommand{\gateThousandP}{45}
\newcommand{\grAgreePct}{97}
\newcommand{\grCommits}{30,998}
\newcommand{\grEteMaxS}{5.2}
\newcommand{\grEteP}{44.3}
\newcommand{\grGateP}{0.75}
\newcommand{\grRepos}{714}
\newcommand{\grStagesMax}{150}
\newcommand{\grUnsafe}{1,237}

\newcommand{\hyAny}{54}
\newcommand{\hyEntry}{6}
\newcommand{\hyImport}{74}
\newcommand{\hyRepos}{711}
\newcommand{\hyReposAny}{469}
\newcommand{\hyReposAnyPct}{66}
\newcommand{\hyStages}{3,969}

\newcommand{\medianCommits}{23}
\newcommand{\medianStages}{4}
\newcommand{\nCandidates}{3,524}
\newcommand{\nCloned}{2,087}
\newcommand{\nCommits}{38,083}
\newcommand{\nPyStage}{1,644}
\newcommand{\nRepos}{714}
\newcommand{\nRootDvc}{2,130}
\newcommand{\nTenCommits}{726}
\newcommand{\nTwoStage}{1,880}
\newcommand{\pbMem}{97}
\newcommand{\pbPNinetyFive}{9.9}
\newcommand{\planThousandP}{1}

\newcommand{\reposAffected}{218}
\newcommand{\reposAffectedPct}{31}
\newcommand{\robCodeHi}{0.5}
\newcommand{\robCodeLo}{0.1}
\newcommand{\robHi}{4.8}
\newcommand{\robLo}{3.0}
\newcommand{\roleDataStale}{2.0}
\newcommand{\roleEvalShare}{29}
\newcommand{\roleEvalSlots}{15}
\newcommand{\roleEvalStale}{4.9}
\newcommand{\roleOtherStale}{1.7}
\newcommand{\roleTrainStale}{2.8}

\newcommand{\shDetTen}{46}
\newcommand{\shFixTen}{22}
\newcommand{\shFixTenMat}{12}
\newcommand{\shRho}{10}
\newcommand{\shSwitchTen}{55}
\newcommand{\shSwitchTenMat}{41}
\newcommand{\shUpfront}{87}
\newcommand{\shadowQ}{0.35}
\newcommand{\tpCandTwelve}{3,580}
\newcommand{\tpOne}{21,849}
\newcommand{\tpTwelve}{179,010}

\newcommand{\unsafeAllCode}{10.5}
\newcommand{\unsafeAllDecl}{13.2}
\newcommand{\unsafeAllEnv}{0.4}
\newcommand{\unsafeAllFail}{0.4}
\newcommand{\unsafeCode}{0.4}
\newcommand{\unsafeCodeHi}{0.7}
\newcommand{\unsafeCodeLo}{0.3}
\newcommand{\unsafeDecl}{3.3}
\newcommand{\unsafeDeclHi}{4.5}
\newcommand{\unsafeDeclLo}{2.4}
\newcommand{\unsafeEnv}{0.4}
\newcommand{\unsafeEnvHi}{0.6}
\newcommand{\unsafeEnvLo}{0.2}
\newcommand{\unsafeFail}{0.4}
\newcommand{\unsafeFailHi}{0.6}
\newcommand{\unsafeFailLo}{0.2}

\journal{Future Generation Computer Systems}
\makeatletter
\def\ps@pprintTitle{%
  \let\@oddhead\@empty
  \let\@evenhead\@empty
  \def\@oddfoot{\hfil\thepage\hfil}%
  \let\@evenfoot\@oddfoot}
\makeatother

\begin{document}

\begin{frontmatter}

\title{Reusing Assurance Evidence in AI Pipelines: Validity Conditions, a Fail-Closed Release Gate, and a Replay of \nRepos{} Public Data Version Control Pipelines}

\author[inst1]{Atanu Mukhopadhyay}
\author[inst2]{Shravya Reddy Nalla}
\affiliation[inst1]{organization={Independent Researcher},
            state={New Jersey},
            country={USA}}
\affiliation[inst2]{organization={Independent Researcher},
            country={USA}}

\begin{abstract}
Rerunning every governance control after every change to an artificial intelligence (AI) pipeline is expensive, so teams reuse earlier results. Reuse is safe only if each control's declared dependencies cover everything that can change its outcome, and only if the release gate checks that reused evidence describes the candidate being released. This article makes both conditions explicit and measures how often the first fails in practice. It defines when a signed evidence record is valid for a candidate and gives sufficient conditions under which change-local re-execution reaches the same release decision as rerunning every control. A reference gate implementing these checks with Ed25519 digital signatures blocks eight of nine manipulation scenarios, against one for a gate that trusts the planner and three for a gate that only verifies signatures; a dependency the policy never declared passes every gate. A replay of \nCommits{} commits from \nRepos{} public Data Version Control (DVC) pipelines measures that failure mode. Reusing stage results according to declared dependencies avoids \avoidDecl\% of stage executions, but in \unsafeDecl\% of commits it reuses a stage whose imported code or referenced configuration changed, and \reposAffectedPct\% of pipelines are affected. A manual audit confirms the undeclared dependency in \auditReal{} of \auditN{} sampled events. Evaluation stages, whose outputs are the evidence, are the most exposed. Adding each stage's static import closure lowers stale reuse to \unsafeCode\% while still avoiding \avoidCode\% of executions; sampled shadow re-execution removes far less, and rerunning everything after unexplained changes avoids only \avoidFail\%.
\end{abstract}


\end{frontmatter}

\section{Introduction}\label{sec:intro}

An AI release is assembled from data, feature code, training code, configuration, model weights, evaluation sets, and the environment that runs them. Governance controls examine that assembly: data validation, performance and fairness evaluation, robustness and privacy tests, security scans, and provenance checks. Running all of them after every change is defensible but costly, and much of the cost buys nothing when a change touches only documentation or a component that a control never reads. Pipelines therefore reuse earlier results, sometimes explicitly through stage caches such as DVC~\cite{dvc2026} and the caching built into ML platforms~\cite{baylor2017tfx}, and sometimes by convention.

Reuse can fail in two different ways. The reuse decision can be wrong because the control depends on something its declaration omits, so a cached result describes a computation that no longer exists. Or the evidence presented at release can differ from what it claims to be: produced for an earlier artifact version, copied from another deployment, edited after signing, signed by the team whose work it certifies, or produced under a superseded policy. Recent systems make governance requirements and evidence machine-checkable~\cite{muhammad2026auditascode,paul2026governanceascode,sharma2026okb,wen2026saco}. To our reading of their published descriptions, none of them defines when a cached result may stand in for a fresh one after the pipeline changes, and no study we are aware of measures how often the dependency declarations of real ML pipelines are complete enough for reuse to be safe.

This article addresses both problems and asks four questions.
\begin{itemize}
\item[] \textbf{RQ1.} Which manipulations of reused evidence does a gate that checks evidence validity stop, and which does it not?
\item[] \textbf{RQ2.} In real ML pipelines, how much re-execution does dependency-based reuse avoid, and how often does it reuse a result whose undeclared inputs changed?
\item[] \textbf{RQ3.} How do alternative invalidation rules, including sampled shadow re-execution, trade avoided work against stale reuse, and what does verification cost at scale?
\item[] \textbf{RQ4.} When a stage is reused stale, does executing it again actually produce different output?
\end{itemize}

The contributions are as follows. First, a validity model for signed assurance evidence and sufficient conditions under which change-local re-execution is decision-equivalent to exhaustive re-execution (Section~\ref{sec:model}). The model separates the planner, which only decides what to run, from the gate, which decides what may be released and never consults the planner. Second, a reference implementation of the planner and gate with Ed25519 signatures~\cite{bernstein2012ed25519,rfc8032}, evaluated against nine manipulation scenarios and two weaker gate designs that are common in practice (Section~\ref{sec:proto}). Third, a replay of \nCommits{} first-parent commits from \nRepos{} public DVC pipelines that compares five invalidation rules against a static reference of each stage's code and configuration dependencies (Sections~\ref{sec:design}--\ref{sec:results}). The replay is checked three ways: a manual audit of sampled events, sandboxed re-execution of flagged stages before and after the change, and a case study of a public-sector valuation pipeline whose maintainers later corrected the same gap. Fourth, an evaluation of sampled shadow re-execution on the real sequence of stale-reuse events, and measurements of planning and verification cost, including parallel throughput.

The study does not claim that static analysis recovers every semantic dependency, and the public repositories it mines are not a sample of industrial pipelines. Its narrower conclusion is that the completeness of dependency declarations, rather than cryptography or gate logic, determines whether evidence reuse is safe, and that inexpensive automatic augmentation of those declarations recovers most of the lost safety.

\section{Background and related work}\label{sec:related}

\subsection{Executable governance and assurance evidence}

Risk-management frameworks organize AI governance around mapping, measuring, and managing risk~\cite{nist2023airmf,iso42001}, and documentation artifacts record evaluation context for models, datasets, and services~\cite{mitchell2019modelcards,gebru2021datasheets,arnold2019factsheets}. Internal auditing frameworks connect these artifacts to an end-to-end review process~\cite{raji2020accountability}. Several 2026 systems make the requirements themselves executable. Audit-as-Code maps governance requirements to versioned policy and executable checks over evidence bundles, with risk-tiered gate decisions in CI/CD~\cite{muhammad2026auditascode}. Governance-as-Code defines 43 machine-checkable acceptance criteria for EU AI Act obligations, implemented as Rego policies, and validates them on two enterprise deployments~\cite{paul2026governanceascode}. Ontological Knowledge Blocks compile regulatory obligations into SHACL constraints over evidence graphs with PROV-O provenance and evaluate them in an AI-assisted HPC resource-allocation scenario~\cite{sharma2026okb}. SACO represents assurance context, its provenance, and explicit information gaps as authoritative semantic structures~\cite{wen2026saco}. These systems establish that evidence can be checked mechanically. The question studied here is different: when an existing evidence record may be reused for a new candidate, and how often the inputs that decide that question are declared incompletely.

\subsection{Provenance and attestation}

W3C PROV and OpenLineage describe entities, activities, and derivations for data and pipeline lineage~\cite{w3c2013provo,openlineage2026}. Software supply-chain systems sign statements about who produced an artifact from which inputs: in-toto chains signed step metadata~\cite{torresarias2019intoto}, SLSA grades the trustworthiness of build provenance~\cite{slsa2023}, and Sigstore makes signing and verification routine for open-source artifacts~\cite{newman2022sigstore}. The evidence records in this article follow the same pattern for governance controls. They bind a measurement to the digests of its inputs, the policy, and the deployment context, and they are signed by an executor that the policy authorizes for that control.

\subsection{Incremental computation and its correctness}

Reusing a result when its inputs are unchanged is the central idea of incremental builds and of regression test selection. Build systems can be characterized by how they decide what to rebuild~\cite{mokhov2018build}. Safe regression test selection reruns every test whose dependencies a change touches~\cite{rothermel1997safe}, and dynamic dependency tracking makes this practical at scale~\cite{gligoric2015ekstazi}. The correctness of both depends on complete dependency information, and that information is often wrong. Unspecified dependencies in make-based builds have been studied empirically~\cite{bezemer2017unspecified}, build fuzzing uncovered faulty build rules in 31 open-source projects~\cite{licker2019detecting}, and Rattle avoids declared dependencies altogether by tracing them~\cite{spall2020rattle}. This article transfers that question from build outputs to assurance evidence in ML pipelines, where a missed dependency leaves a governance decision resting on a result for a different computation.

\subsection{ML pipelines and pipeline-versioning tools}

Production ML depends on automated pipelines, data validation, and tests~\cite{sculley2015debt,breck2017mltestscore,polyzotis2017data,schelter2018deequ,kreuzberger2023mlops}, and asset-management tools track the data, models, and pipelines involved~\cite{idowu2022asset}. DVC records pipeline stages in a \texttt{dvc.yaml} file that lists each stage's command, dependencies, parameters, and outputs, and reruns a stage only when its definition, a declared dependency, or a declared parameter changes~\cite{dvc2026}. Barrak et al.\ studied how DVC files co-evolve with source code in 391 GitHub projects and found tight coupling between them~\cite{barrak2021dvc}. A large-scale study of Jupyter notebooks found that few of them reproduce their stored results~\cite{pimentel2019notebooks}. We are not aware of work that measures whether DVC dependency declarations cover what stages actually read, or what that implies for reusing evidence. Table~\ref{tab:positioning} summarizes the positioning.

\begin{table*}[t]
\centering
\caption{Positioning relative to the closest lines of work.}\label{tab:positioning}
\small
\begin{tabularx}{\textwidth}{@{}>{\raggedright\arraybackslash}X*{4}{>{\centering\arraybackslash}p{2.1cm}}>{\raggedright\arraybackslash}p{2.3cm}@{}}
\toprule
Line of work & Checkable evidence & Reuse validity conditions & Independent fail-closed gate & Measures declaration completeness & Domain \\
\midrule
Executable governance~\cite{muhammad2026auditascode,paul2026governanceascode,sharma2026okb,wen2026saco} & Yes & No & Yes & No & AI governance \\
Supply-chain attestation~\cite{torresarias2019intoto,slsa2023,newman2022sigstore} & Yes & Partly & Yes & No & Software artifacts \\
Build and test dependency studies~\cite{bezemer2017unspecified,licker2019detecting,spall2020rattle} & No & Yes & No & Yes & Builds \\
DVC co-evolution~\cite{barrak2021dvc} & No & No & No & No & ML pipelines \\
This article & Yes & Yes & Yes & Yes & ML pipelines \\
\bottomrule
\end{tabularx}
\end{table*}

\section{Problem and threat model}\label{sec:threat}

A \emph{candidate} is a proposed release. It consists of a set of artifacts identified by content digests, a deployment context (intended use, deployment setting, risk tier), and the digest of the policy it was built under. A \emph{control} is a check that produces an \emph{evidence record}: a measurement and a pass or fail status, bound to the digests of the artifacts the control read, the policy digest, the context digest, a validity period, and the identity of the executor that produced it. The policy lists, for each risk tier, the controls that are mandatory, each control's declared dependencies, its maximum evidence age, and the executor keys allowed to sign its evidence. After a change, a \emph{planner} selects controls to run, and a \emph{gate} decides whether the candidate may be released.

The threat model covers faults and misbehavior inside an organization that runs its own pipeline. The planner may be buggy or misconfigured. Teams may present evidence from an earlier artifact version or another deployment, edit stored records, sign their own results, rely on expired evidence, release under a policy that changed during the run, or declare a lower risk tier to shrink the mandatory set. A dependency may simply be missing from a control's declaration, with no intent to deceive. The model assumes that the signing keys of authorized executors are not compromised, that an authorized executor reports what it measured, and that the gate's code and the active policy are protected. Attacks on those components need the broader supply-chain controls cited above and are out of scope.

\section{Evidence validity model}\label{sec:model}

\subsection{Validity and admissibility}

Let $a$ be a candidate with artifact digests $\nu_a(\cdot)$, context $\chi_a$, policy digest $\pi_a$, and risk tier $r_a$. Let $P$ be the active policy with digest $\pi_P$, mandatory sets $M_P(r)$, declared dependencies $D_P(k)$, maximum ages $\tau_P(k)$, and authorized signers $S_P(k)$ for each control $k$.

\begin{definition}[Valid evidence]
An evidence record $q$ for control $k$ is \emph{valid} for candidate $a$ at time $t$ if all of the following hold: (i) its signature verifies under a key in $S_P(k)$; (ii) its contract digest equals the digest of $k$'s contract in $P$ (its dependencies, maximum age, and authorized signers); (iii) its context digest equals the digest of $\chi_a$; (iv) its recorded dependency set equals $D_P(k)$ and, for every $d\in D_P(k)$, its recorded digest equals $\nu_a(d)$; (v) $t \le q.t_{\mathrm{exec}}+\tau_P(k)$; and (vi) its status is pass.
\end{definition}

Binding a record to its own control's contract, rather than to the whole policy, means that editing one control's contract does not invalidate the evidence of the others; admissibility still binds the candidate as a whole to the active policy. Condition (iv) makes validity relative to the \emph{declared} dependencies. Equality of digests is the default. A control-specific equivalence, for example one that ignores formatting changes in a configuration file, must be part of the policy and is itself evidence that can be reviewed.

\begin{definition}[Admissibility]
Candidate $a$ is \emph{admissible} at time $t$ if $\pi_a=\pi_P$ and, for every $k\in M_P(r_a)$, the bundle presented to the gate contains a record valid for $a$ at $t$.
\end{definition}

The risk tier enters the mandatory set, so it cannot be self-declared. The prototype treats risk classification as a mandatory control whose evidence binds the context digest. A candidate that claims a lower tier presents a context whose digest no longer matches that record.

\subsection{Decision equivalence}

Let $\Phi_k(a)$ be the decision of control $k$ on candidate $a$, and let $D^{*}(k)$ be the smallest set of inputs (artifacts, configuration, code, environment) whose values can change $\Phi_k$. A declaration is \emph{sound} for $k$ if $D^{*}(k)\subseteq D_P(k)$. Exhaustive evaluation $F(a)$ runs every mandatory control. Incremental evaluation $I(a)$ reuses a stored record for $k$ whenever it is valid for $a$ and runs $k$ otherwise.

\begin{proposition}[Decision equivalence]\label{prop:equiv}
$I(a)$ and $F(a)$ return the same release decision if (1) the declaration is sound for every mandatory control; (2) every value in each $D_P(k)$ is versioned by digest or compared under an explicit equivalence; (3) each control's decision is a deterministic function of its declared inputs, or exhaustive and incremental evaluation apply the same statistical decision rule; and (4) the gate implements admissibility.
\end{proposition}

\begin{proof}[Proof sketch]
Partition the mandatory controls into those with a valid stored record and the rest. The rest are re-executed under both strategies. For a control with a valid record, condition (iv) of validity and condition (2) give equal declared inputs; soundness extends this to every input that can affect the decision, and condition (3) then gives the same decision. Both strategies therefore present the same per-control decisions to the same admissibility predicate. If soundness fails, an input outside $D_P(k)$ can change $\Phi_k$ while the stored record remains valid, which is a counterexample.
\end{proof}

\begin{proposition}[Planner independence]\label{prop:planner}
If the gate implements admissibility and derives $M_P(r_a)$ from the active policy, then no planner output can cause the release of a candidate that lacks a valid record for some mandatory control. Planner faults can only cause extra executions or holds.
\end{proposition}

\begin{proof}
Admissibility refers only to the candidate, the active policy, and the presented bundle. The planner's output does not appear in it.
\end{proof}

\begin{remark}
Validity is checked against $D_P(k)$, not $D^{*}(k)$. When a declaration is unsound, a record can be valid and still describe a superseded computation, and no gate can tell. The defense is mapping assurance: deriving dependencies automatically where possible, failing closed on changes nobody declared, and sampling reused controls for shadow re-execution. If a fraction $\rho$ of reused controls is re-executed and compared, an unsound declaration that is exercised in $n$ releases escapes detection with probability $(1-\rho)^n$.
\end{remark}

\subsection{Planner}

Given a candidate $a$ with predecessor $a'$, the planner runs
\begin{equation}
\begin{aligned}
K_{\mathrm{exec}} = \Big(&\{k : D_P(k)\cap\Delta(a,a')\neq\emptyset\}\\
&\cup \{k\in M_P(r_a): \nexists\,q \text{ valid}\} \cup K_{\mathrm{trig}}\Big)\cap M_P(r_a),
\end{aligned}
\label{eq:plan}
\end{equation}
where $\Delta(a,a')$ is the set of artifacts whose digests differ and $K_{\mathrm{trig}}$ contains controls triggered by runtime events. If $\Delta$ contains an artifact that no control declares, the planner fails closed and runs every mandatory control. Expired evidence is covered by the second term.

\section{Reference implementation}\label{sec:proto}

The prototype is written in Python 3.14 and uses the \texttt{cryptography} package for Ed25519. It has three parts: signed evidence records, a planner implementing Eq.~(\ref{eq:plan}), and a gate implementing admissibility (Figure~\ref{fig:arch}). Records are serialized as canonical JSON (sorted keys, no insignificant whitespace), and each executor holds its own signing key. The policy bundle and each control's contract are canonicalized and hashed, so every record names the exact contract it was produced under and every candidate the exact policy. The prototype and the analysis scripts were written with the assistance of AI coding tools (OpenAI Codex and Anthropic Claude); the authors specified the designs and analyses, reviewed the code, and checked the reported values against the result files. The gate indexes the presented bundle by control, re-derives the mandatory set from the active policy and the candidate's context, and returns release or hold with a reason for every unmet requirement. It never reads the planner's output.

\begin{figure*}[t]
\centering
\includegraphics[width=0.66\textwidth]{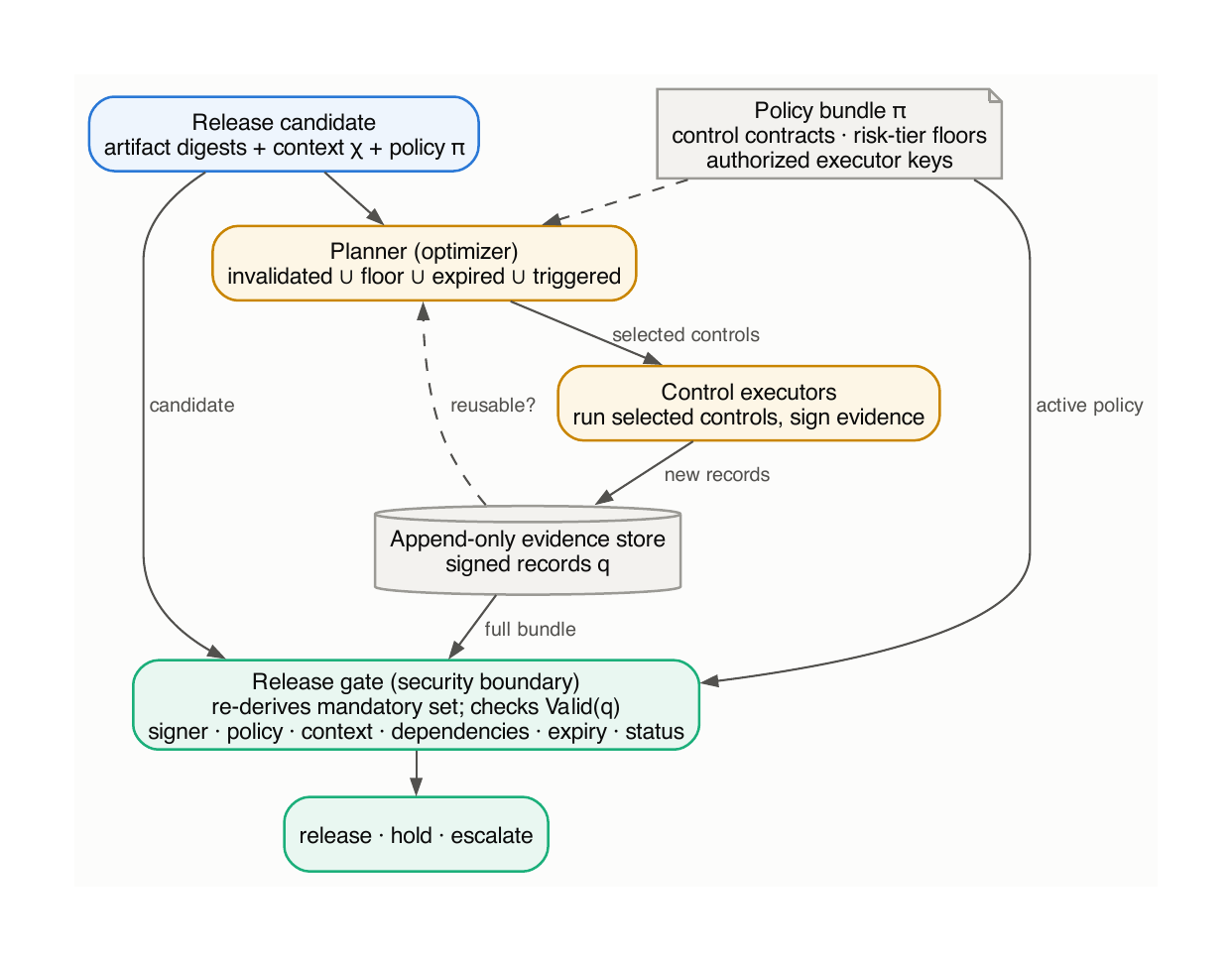}
\caption{Separation of the planner and the gate. The planner decides which controls to run and may reuse stored evidence. The gate re-derives the mandatory set from the active policy and checks every presented record for signer authorization, policy and context binding, dependency digests, freshness, and status.}\label{fig:arch}
\end{figure*}

For comparison, the gate can run in two weaker modes that resemble common practice. A \emph{planner-trusting} gate releases when every control the planner scheduled passed, as in a CI pipeline that promotes when all jobs are green. A \emph{signature-only} gate verifies that each mandatory control has a signed, passing record and that the candidate names the active policy, but does not check who signed each record, under which contract or context it was produced, or for which artifact versions.

\section{Evaluation design}\label{sec:design}

\subsection{RQ1: manipulation scenarios}

Each scenario starts from a release with complete, valid evidence, applies a change, and then applies one manipulation (Table~\ref{tab:scenarios}). The policy has nine controls, three risk tiers, and one authorized executor per control. A separate model-team key is registered in the keyring but not authorized for any assurance control. Two scenarios are legitimate reuse and must release; one is an honest failure and must hold; nine are manipulations or faults that must hold.

\subsection{RQ2--RQ3: replay of public DVC pipelines}

\paragraph{Corpus} On 13 October 2026 we collected candidate repositories from two sources: the results of a Sourcegraph code search for files named \texttt{dvc.yaml}, and GitHub repository searches for DVC- and MLOps-related topics and README text. Forks and archived repositories were excluded. Of \nCandidates{} candidates, \nRootDvc{} had a \texttt{dvc.yaml} at the root of the default branch and \nCloned{} could be cloned. The filters in Table~\ref{tab:corpus} were fixed before any result was examined. Repositories sharing a root commit are copies of the same history, and only the copy with the most analyzable commits was kept. The final corpus has \nRepos{} repositories and \nCommits{} commits; the median repository contributes \medianCommits{} commits and defines \medianStages{} stages.

\begin{table}[t]
\centering
\caption{Corpus construction.}\label{tab:corpus}
\small
\begin{tabular}{@{}lr@{}}
\toprule
Step & Repositories \\
\midrule
Candidates from search & \nCandidates \\
Root \texttt{dvc.yaml} on default branch & \nRootDvc \\
Cloned successfully & \nCloned \\
At least two stages at HEAD & \nTwoStage \\
At least one resolvable Python stage & \nPyStage \\
At least 10 analyzable commits & \nTenCommits \\
One repository per root commit & \nRepos \\
\midrule
Commits analyzed & \nCommits \\
\bottomrule
\end{tabular}
\end{table}

\paragraph{Replay} Each pipeline stage stands for a control whose evidence is its output. For every first-parent commit after \texttt{dvc.yaml} first appears, the replay compares the commit with its parent and decides which stages must re-execute. It emulates DVC's own rule. A stage is invalidated if a declared dependency changed (a file, a directory prefix, or a path templated from parameters), if a declared parameter key changed value, or if the stage definition changed after resolving \texttt{\$\{\ldots\}} templates against \texttt{params.yaml}. Invalidation propagates to every stage that consumes an invalidated stage's outputs, and a change to a \texttt{.dvc} pointer counts as a change to the data it tracks.

\paragraph{Reference dependencies} For each stage, static analysis resolves the entry points in its command: \texttt{python script.py}, \texttt{python -m module}, and script or notebook files named on the command line. It then computes the transitive closure of local Python imports, including package \texttt{\_\_init\_\_} modules and relative imports. String literals in that closure that name configuration files in the repository (YAML, JSON, TOML, INI) are added; a change to such a file counts only if a changed leaf key is named somewhere in the stage's code. A change to Python code counts only if the parsed syntax tree changes after docstrings are removed, so comment and formatting edits are ignored, and a change to a notebook counts only if a code cell changes. Changes to environment files (requirements, lock files, \texttt{pyproject.toml}, Dockerfiles) are reported separately, because they plausibly affect every stage but their relevance cannot be decided statically. The reference set is a lower bound on true dependencies. It misses dynamic imports, files opened through computed paths, and data read outside the repository.

\paragraph{Invalidation rules} Five rules are compared. \emph{Full} reruns every stage after every commit. \emph{Declared} is DVC's rule. \emph{Declared+imports} adds each stage's import closure to its declared dependencies. \emph{Declared+imports+environment} also reruns every Python stage when an environment file changes. \emph{Fail-closed} applies the declared rule but reruns every stage when a changed code, configuration, notebook, or environment file is declared by no stage, which is the conservative fallback of Eq.~(\ref{eq:plan}). For each rule we report the share of stage slots executed and the share of commits in which at least one stage that the reference marks as invalidated is reused (a \emph{stale reuse}). Intervals are 95\% percentile intervals from 2,000 bootstrap resamples of repositories.

\paragraph{Stage roles and change types} To see where evidence goes stale, each stage was assigned a role from its name and command: evaluation (names containing eval, test, metric, score, benchmark, assess, report, or compare), training (train, fit, tune, hyperparameter search, fine-tuning), data preparation (ingest, load, download, prepare, clean, split, feature, transform, embed, index, validate), or other. Commits were grouped by the categories of files they changed.

\paragraph{Audit} From the commits in which the declared rule reuses an invalidated stage, we drew a random sample of \auditN{} and examined each flagged stage, the import chain from its entry point to the changed file, and the diff. Each event was labeled \emph{material} (a real undeclared dependency whose change plausibly alters the stage's output or measurement), \emph{benign} (a real undeclared dependency whose change does not, for example a new unused constant or a type annotation), or \emph{analyzer error} (no real dependency, or DVC would in fact rerun the stage). An earlier audit round on a separate sample was used only to correct the analyzer before the final run. After the main results, we also measured declaration hygiene at each repository's HEAD (whether each stage declares its entry script and its local imports) and recomputed the main metrics on corpus subsets. These analyses were not pre-specified.

\paragraph{Shadow re-execution} Shadow runs re-execute a sample of reused stages and compare their outputs with the cached ones. We replayed each repository's history in first-parent order. At every commit, each stage that the declared rule reused was shadow-executed with probability $\rho\in\{1,5,10,20\}\%$. A shadow run of a stale stage exposes the problem with probability $q$: $q=1$ if any stale input changes the output, or $q=\shadowQ$, the audited share of material changes. Two responses to an exposure were compared: declaring the dependency that was missed, or switching the repository to declarations augmented with each stage's import closure. Each configuration was repeated 200 times.

\subsection{RQ3: cost}

Gate verification time was measured for 9 to 1,000 mandatory controls (200 repetitions each). The full prototype was also run over the commit histories of the corpus (up to 400 commits per repository): each stage becomes a control whose dependencies are its declared dependency paths, parameter keys, and upstream outputs; at each commit the planner selects stages, the selected stages' executors sign fresh records, and the gate admits the candidate from the merged bundle. Parallel throughput was measured by admitting 2,000 candidates, each with 50 signed records, with one gate per worker process and candidates divided among 1 to 12 processes. The planner core was also benchmarked on synthetic sparse graphs of up to 500,000 nodes. All timings were taken on one Apple Silicon laptop (14 logical CPUs) running Python 3.14.

\subsection{RQ4: re-execution of flagged stages}\label{sec:rerundesign}

The audit judges whether a change could alter a stage's output; re-execution measures it. For commits in which the declared rule reused a stage that the reference marks as invalidated, the stage was executed twice at the parent commit and once at the flagged commit, and its declared outputs were compared. If the two parent runs agree and the commit run differs, the reused output was not what a fresh run produces. A flagged stage was eligible if it and every stage upstream of it (at most six stages) run a single Python entry point, every input that no stage in that chain produces is a file in the repository (no data held in a DVC remote), the code does not contact the network, and the environment names no GPU-only packages. Upstream stages were executed in order before the flagged stage, and at most three events were taken per repository. Each repository's environment was installed with uv into its own Python 3.11 environment from \texttt{requirements.txt} or \texttt{pyproject.toml}, together with a small common base (NumPy, pandas, scikit-learn, SciPy, PyYAML, Matplotlib, seaborn, joblib). If the requirements could not be resolved together, each line was installed on its own; if a run failed on a missing module, that module was installed and the run retried, at most twice. Because the code comes from unvetted public repositories, every stage run executed under the macOS sandbox with networking disabled, reads of the user's home directory denied, and writes confined to a scratch directory, with a three-minute limit per run. Dependency installation needs the network and ran outside the sandbox. The protocol was revised once. A first version, which required a requirements file, ran no upstream stages, and allowed no deep-learning frameworks, completed too few events to be informative; we report its outcome alongside the final one.

\section{Results}\label{sec:results}

\subsection{RQ1: which manipulations the gate stops}

Table~\ref{tab:scenarios} reports the decisions. The full gate released both legitimate candidates, held the honest failure, and held eight of the nine manipulations. The planner-trusting gate held only expired evidence, because a correct planner re-ran the expired control and the fresh run failed. The signature-only gate also held tampered records, missing evidence, and candidates built under a superseded policy. Every gate released the undeclared-dependency scenario (S10): the stored fairness record was valid for the declared inputs, and the cohort definition that changed was not among them. This is the limit stated in the remark of Section~\ref{sec:model}, and it is the failure mode that the replay measures.

\begin{table*}[t]
\centering
\caption{Gate decisions in the manipulation scenarios. R = release, H = hold. Correct decisions are marked with an asterisk.}\label{tab:scenarios}
\small
\begin{tabularx}{\textwidth}{@{}lXccc@{}}
\toprule
 & Scenario & Planner-trusting & Signature-only & Full gate \\
\midrule
S0 & Documentation-only change. Only provenance is invalidated and re-run; other evidence is reused. & R$^{*}$ & R$^{*}$ & R$^{*}$ \\
S1 & Dataset change, correct re-run. Invalidated controls re-run and pass. & R$^{*}$ & R$^{*}$ & R$^{*}$ \\
S2 & Planner omits an invalidated control. Fairness evidence for the old dataset is presented. & R & R & H$^{*}$ \\
S3 & Cross-context replay. Evidence certified for a different deployment is reused. & R & R & H$^{*}$ \\
S4 & Tampered evidence record. A failing privacy result is edited to pass after signing. & R & H$^{*}$ & H$^{*}$ \\
S5 & Self-certified result. The model team signs its own robustness pass. & R & R & H$^{*}$ \\
S6 & Expired evidence. Evidence older than its maximum age is presented. & H$^{*}$ & R & H$^{*}$ \\
S7 & Policy changed before release. The candidate was planned and evaluated under a policy that has since changed. & R & H$^{*}$ & H$^{*}$ \\
S8 & Risk-tier downgrade. A tier-3 system is submitted as tier 1. & R & R & H$^{*}$ \\
S9 & Mandatory control never run. Explainability evidence is absent. & R & H$^{*}$ & H$^{*}$ \\
S10 & Undeclared semantic dependency. Cohort definition changes but is absent from fairness dependencies. & R & R & R \\
S11 & Control fails honestly. A re-run fairness control fails. & H$^{*}$ & H$^{*}$ & H$^{*}$ \\
\midrule
\multicolumn{2}{@{}l}{Manipulations held (S2--S10)} & 1/9 & 3/9 & 8/9 \\
\bottomrule
\end{tabularx}
\end{table*}

\subsection{RQ2: reuse and stale reuse in real pipelines}

The commits span the usual work of a pipeline project: 34\% touched code, 21\% configuration or parameter files, 11\% DVC data pointers, 12\% environment files, and 7\% the pipeline definition itself, while 18\% changed only documentation. Under the declared rule, 71\% of commits re-executed no stage at all.

Table~\ref{tab:rules} and Figure~\ref{fig:frontier} compare the five rules. Under DVC's declared rule, pipelines executed \execDecl\% of stage slots, so \avoidDecl\% of stage results were reused. In \unsafeDecl\% of commits (95\% interval \unsafeDeclLo--\unsafeDeclHi\%), at least one reused stage had an imported module, an entry script, or a referenced configuration key that changed. Counting environment-file changes as well raises the share to \unsafeAllDecl\%. Stale reuse is concentrated: \reposAffected{} of the \nRepos{} pipelines (\reposAffectedPct\%) were affected at least once, and the median pipeline was never affected. Among stages that record metrics or plots, the outputs most likely to be used as evaluation evidence, the declared rule reused a stale stage in \evalUnsafeDecl\% of commits. Code dependencies (an imported module or the stage's own entry script) caused \causeCodePct\% of the flagged commits; configuration keys read through a string reference caused the rest.

The declarations themselves explain why. At each repository's HEAD, \hyAny\% of the \hyStages{} stages with a Python entry point had at least one undeclared local code dependency, and \hyReposAny{} of \hyRepos{} pipelines (\hyReposAnyPct\%) contained such a stage. \hyEntry\% of stages did not declare their own entry script, and \hyImport\% of the stages that import local modules left at least one of them undeclared. Most of these gaps are harmless most of the time, because the undeclared file rarely changes without a declared one; the replay measures how often it does. Table~\ref{tab:robust} shows that the result is not driven by a few large repositories or by abandoned projects. Across the subsets, the declared rule reused a stale stage in \robLo--\robHi\% of commits, and the import-closure rule in \robCodeLo--\robCodeHi\%.

\begin{table*}[t]
\centering
\caption{Sensitivity of the replay results to the corpus. Executed and stale-reuse shares as in Table~\ref{tab:rules}; brackets are 95\% bootstrap intervals over repositories.}\label{tab:robust}
\small
\begin{tabularx}{\textwidth}{@{}Xrr*{3}{>{\centering\arraybackslash}p{2.7cm}}@{}}
\toprule
Subset & Repositories & Commits & Declared: executed (\%) & Declared: stale reuse (\%) & Declared + imports: stale reuse (\%) \\
\midrule
All repositories & 714 & 38,083 & 16.6 & 3.3 [2.4, 4.5] & 0.4 \\
At least 50 analyzed commits & 145 & 25,992 & 16.6 & 3.6 [2.3, 5.4] & 0.4 \\
At least 5 stages at HEAD & 330 & 19,355 & 13.4 & 4.8 [3.3, 6.7] & 0.5 \\
At least 5 GitHub stars & 22 & 5,343 & 46.4 & 3.0 [0.6, 10.3] & 0.1 \\
Last commit in 2026 & 548 & 28,696 & 12.5 & 3.0 [2.0, 4.2] & 0.3 \\
Largest 1\% of repositories removed & 707 & 27,768 & 15.3 & 4.2 [3.2, 5.2] & 0.5 \\
Repository-weighted mean & 714 & 38,083 & 18.0 & 3.2 & 0.5 \\
\bottomrule
\end{tabularx}
\end{table*}

\begin{table*}[t]
\centering
\caption{Invalidation rules on \nCommits{} commits from \nRepos{} DVC pipelines. Executed: share of stage slots re-executed. Stale reuse: share of commits in which at least one stage invalidated under the reference is reused, counting code and configuration dependencies (environment changes in the last column). Brackets: 95\% bootstrap intervals over repositories.}\label{tab:rules}
\small
\begin{tabular}{@{}lccc@{}}
\toprule
Rule & Executed (\%) & Stale reuse, code/config (\%) & Stale reuse, incl.\ environment (\%) \\
\midrule
Full & 100.0 & 0.0 & 0.0 \\
Declared (DVC) & \execDecl{} [\execDeclLo, \execDeclHi] & \unsafeDecl{} [\unsafeDeclLo, \unsafeDeclHi] & \unsafeAllDecl \\
Declared + imports & \execCode{} [\execCodeLo, \execCodeHi] & \unsafeCode{} [\unsafeCodeLo, \unsafeCodeHi] & \unsafeAllCode \\
Declared + imports + environment & \execEnv{} [\execEnvLo, \execEnvHi] & \unsafeEnv{} [\unsafeEnvLo, \unsafeEnvHi] & \unsafeAllEnv \\
Fail-closed fallback & \execFail{} [\execFailLo, \execFailHi] & \unsafeFail{} [\unsafeFailLo, \unsafeFailHi] & \unsafeAllFail \\
\midrule
Reference (code/config) & \execRefCode & 0.0 & --- \\
\bottomrule
\end{tabular}
\end{table*}

\begin{figure}[t]
\centering
\includegraphics[width=\columnwidth]{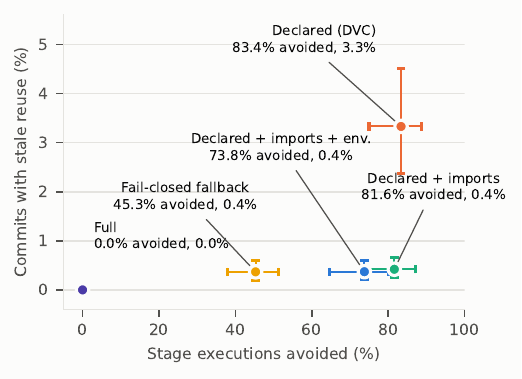}
\caption{Avoided stage executions against commits with stale reuse, for each invalidation rule. Error bars: 95\% bootstrap intervals over repositories.}\label{fig:frontier}
\end{figure}

The audit (Table~\ref{tab:audit}) confirmed a real undeclared dependency in \auditReal{} of \auditN{} sampled events (\auditRealPct\%; Wilson 95\% interval \auditRealLo--\auditRealHi\%). In \auditMaterial{} events (\auditMaterialPct\%; \auditMaterialLo--\auditMaterialHi\%), the change plausibly altered the reused output. Examples include a rewritten system prompt and a changed generation configuration for LLM agents under benchmark, an evaluation metric switched from \texttt{pmin} to \texttt{pmax} in an R script that its stage did not declare, a text-cleaning step re-enabled in an undeclared transformation module, a different model constructor in a training module, and a changed learning-rate warmup schedule. The benign events were dominated by added helper functions and constants, logging and experiment-tracking changes, and refactorings that preserve behavior. All four analyzer errors were metrics files written by another stage, which a shared script named as a string but did not read. Combining the audit with the replay, roughly 1.2\% of commits reused a stage whose change plausibly altered its output.

\begin{table}[t]
\centering
\caption{Manual audit of \auditN{} randomly sampled events in which the declared rule reused an invalidated stage.}\label{tab:audit}
\small
\begin{tabularx}{\columnwidth}{@{}Xrr@{}}
\toprule
Label & Events & Share (\%) \\
\midrule
Material: real dependency, output plausibly changed & \auditMaterial & \auditMaterialPct \\
Benign: real dependency, behavior-neutral change & \auditBenign & \auditBenignPct \\
Analyzer error & \auditError & \auditErrorPct \\
\bottomrule
\end{tabularx}
\end{table}

Table~\ref{tab:breakdown} shows where stale reuse occurs. Evaluation stages, whose outputs are the measurements that governance decisions rest on, were the most exposed: \roleEvalStale\% of reused evaluation stages were stale, against \roleDataStale\% of data-preparation stages, \roleTrainStale\% of training stages, and \roleOtherStale\% of other stages. Evaluation stages held \roleEvalSlots\% of stage slots but accounted for \roleEvalShare\% of stale stage reuses. By change type, \chCodeStale\% of commits that touched code led to stale reuse, against \chConfigStale\% of commits that touched configuration and \chDataStale\% of commits that changed only DVC data pointers; no commit that changed only documentation did.

\begin{table}[t]
\centering
\caption{Stale reuse under the declared rule by stage role (top) and by the type of file a commit changed (bottom).}\label{tab:breakdown}
\small
\begin{tabularx}{\columnwidth}{@{}X*{3}{>{\centering\arraybackslash}p{1.45cm}}@{}}
\toprule
Stage role & Stage slots & Stale (\% of reused) & Share of all stale (\%) \\
\midrule
Data preparation & 75,793 & 2.0 & 32 \\
Training & 33,689 & 2.8 & 18 \\
Evaluation & 31,043 & 4.9 & 29 \\
Other & 68,342 & 1.7 & 21 \\
\midrule
Commit touched & Commits & Stale reuse (\% of commits) & \\
\midrule
Code & 13,047 & 8.9 & \\
Configuration or parameters & 8,168 & 5.7 & \\
Notebook & 1,814 & 5.7 & \\
Tests & 4,683 & 5.9 & \\
Pipeline definition & 2,557 & 6.1 & \\
Environment files & 4,579 & 4.0 & \\
CI configuration & 3,864 & 1.8 & \\
DVC data pointer & 4,122 & 0.9 & \\
Documentation only & 6,930 & 0.0 & \\
\bottomrule
\end{tabularx}
\end{table}

\subsection{A case: the Cook County residential valuation pipeline}\label{sec:case}

The Cook County Assessor's Office publishes the pipeline behind its residential automated valuation model, which produces initial assessed values for single- and multi-family residential properties in Cook County, Illinois~\cite{ccao2026avm}. Its \texttt{dvc.yaml} declares data inputs, parameters, and outputs for each stage but, until August 2024, not the R scripts that the stages run. The project documentation tells users to run the pipeline with \texttt{dvc repro}, noting that each stage runs only if necessary. On 24 January 2024 a commit titled ``Boot 0 or sub-0 values from ratio stats'' changed one line of \texttt{pipeline/03-evaluate.R}, replacing \texttt{pmin} with \texttt{pmax} in the bound applied to estimated values before the assessor-specific ratio statistics (COD, PRD, and PRB) are computed, so that estimates are floored at 1 rather than capped at 1. These statistics measure the uniformity and vertical equity of assessments, and the evaluation stage reports them for every geography and property class. Because the script was not a declared dependency, \texttt{dvc repro} would have reused the previous evaluation outputs, and a gate checking declared dependencies would have accepted them as valid. Over the \ccaoBefore{} commits before the fix, the replay found stale reuse in \ccaoStale{} (\ccaoStalePct\%), including \ccaoEval{} stale reuses of the evaluation stage. In August 2024 the maintainers added each stage's script to its dependencies (``Edit DVC deps to include cmd run files''), and the replay found no stale reuse in the \ccaoAfter{} commits that followed. The case illustrates the failure mode in a maintained production pipeline. It also gives an independent check on the analyzer: the change the maintainers made is the one the analysis recommends, and it removed every flagged event.

\subsection{RQ3: trading avoided work against stale reuse}

Adding each stage's import closure to its declaration (Declared + imports) lowered the share of commits with stale code or configuration reuse from \unsafeDecl\% to \unsafeCode\% and raised executed stage slots only from \execDecl\% to \execCode\%. The fail-closed fallback reached \unsafeFail\%, close to the import-closure rule, but executed \execFail\% of stage slots, because many commits change some code or configuration file that no stage declares. Treating environment changes as dependencies of every Python stage closed most of the remaining environment exposure, at \execEnv\% of slots executed. The residual stale reuse under the augmented rules comes from configuration keys read by string reference, which the import closure does not add.

Shadow re-execution is a weaker substitute (Figure~\ref{fig:shadow}). Re-executing \shRho\% of reused stages and declaring each dependency that a shadow run exposes reduced stale-reuse commits by only \shFixTen\% when every stale input changes the output, and by \shFixTenMat\% when only the audited share of material changes is detectable. Undeclared dependencies keep arriving through new files, so fixing them one at a time lags behind. Using the first exposure as a signal to switch the repository to import-closure declarations did better, reducing stale reuse by \shSwitchTen\% (\shSwitchTenMat\% under the conservative assumption), but still less than adopting the import closure from the start (\shUpfront\%). At \shRho\% sampling, \shDetTen\% of affected repositories had at least one exposure during their history. Shadow runs are best used to detect that declarations are unsound, not to compensate for it.

\begin{figure}[t]
\centering
\includegraphics[width=\columnwidth]{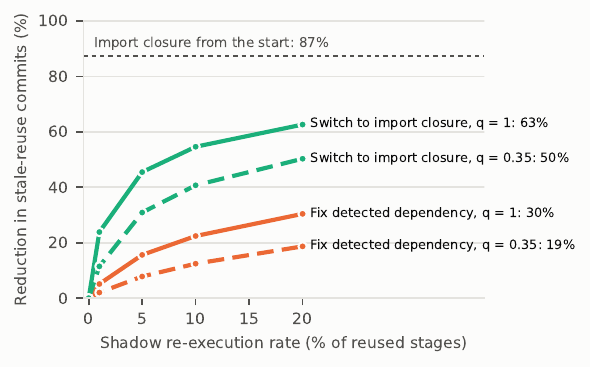}
\caption{Reduction in commits with stale reuse when a share of reused stages is re-executed as shadow runs, for two responses to an exposure and two exposure probabilities $q$ (200 repetitions per point). The dotted line is the reduction from adopting import-closure declarations at the start.}\label{fig:shadow}
\end{figure}

Verification is inexpensive relative to the controls it protects (Figure~\ref{fig:cost}). Gate time grew linearly at about \gatePerRecord{}~$\mu$s per record, dominated by Ed25519 verification: the 95th percentile was \gateNineP{}~ms for nine mandatory controls and \gateThousandP{}~ms for 1,000. The planner selected controls in under \planThousandP{}~ms at 1,000 controls. Replaying the corpus histories through the full prototype covered \grCommits{} commits from \grRepos{} pipelines. The 95th-percentile time to plan, sign, and admit a candidate was \grEteP{}~ms, most of it spent hashing declared directory dependencies; the gate alone took \grGateP{}~ms at the 95th percentile, and the slowest commit, in a pipeline with \grStagesMax{} stages, took \grEteMaxS{}~s. The prototype's planner selected the same number of stages as the replay's declared rule in \grAgreePct\% of commits, and the gate released all \grUnsafe{} commits in which the replay found stale reuse. This confirms on real histories that a gate checking declared dependencies cannot see an undeclared one. Verification parallelizes across candidates: one process verified \tpOne{} records per second and twelve processes \tpTwelve{}, or \tpCandTwelve{} candidates with 50 controls each per second. The planner core processed sparse graphs of 500,000 nodes in \pbPNinetyFive{}~ms at the 95th percentile, with \pbMem{}~MiB of peak traced memory.

\begin{figure}[t]
\centering
\includegraphics[width=\columnwidth]{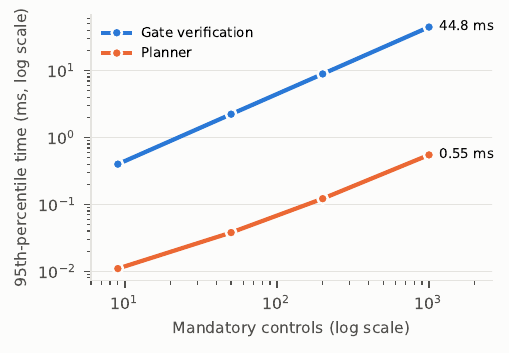}
\caption{Gate verification and planning time against the number of mandatory controls (95th percentile of 200 repetitions).}\label{fig:cost}
\end{figure}

\subsection{RQ4: re-execution of flagged stages}\label{sec:rerun}

Of the 124 flagged events in 39 repositories that met the eligibility conditions, 40 were sampled. Three events, from three repositories, completed all three runs with agreeing parent runs. In one of them the stage's output at the flagged commit differed from its output at the parent commit, and in two it did not. In one further event the stage ran at the parent commit but failed at the flagged commit, so reuse would have certified a stage that no longer ran. Four events gave different outputs in the two parent runs and were set aside. The remaining events could not be reproduced offline from the repository: a file that no declared dependency provides was missing (5); a package the environment did not declare was missing (5); the code raised another error (11); the code tried to reach the network (2); the requirements could not be installed (6); the stage ran but wrote none of its declared outputs (1); the checkout could not be exported (2). Missing inputs and missing packages are incomplete declarations in their own right: a stage whose data dependencies and environment were fully declared would run from a checkout. A first, stricter version of the harness, which required a requirements file and did not run upstream stages, completed one of 22 sampled events; its output changed.

The event whose output changed is a churn-prediction pipeline. Its dataset stage is invoked as \texttt{python src/\allowbreak data/\allowbreak load\_data.py \allowbreak -{}-config=\allowbreak params.yaml}, but \texttt{params.yaml} is neither a declared dependency nor a declared parameter. The flagged commit added 13 features to the feature list in that file, and the stage's output, the raw training table, changed accordingly; under \texttt{dvc repro} the old table would have been reused. One of the two identical outcomes was flagged through a metrics file that another stage writes, which is the pattern the audit classified as an analyzer error, so its identical output is consistent with that classification. The re-execution results are few, and they mostly show how rarely a public pipeline can be re-executed from what it declares. Where re-execution was possible, it confirmed both kinds of event the audit described: stale reuse that changes the evidence, and stale reuse that does not.

\section{Discussion}\label{sec:discussion}

\subsection{What determines whether reuse is safe}

The scenarios and the replay point to the same conclusion from opposite directions. A gate that checks signer authorization, policy and context binding, dependency digests, and freshness stops the manipulations that weaker gates miss, and it does so in milliseconds. It cannot stop reuse that the declarations themselves permit. In the public pipelines, declarations were incomplete often enough to matter: roughly one commit in thirty reused a stage whose code or configuration had changed, and nearly one pipeline in three was affected at least once. Most of these changes were behavior-neutral, which explains why incomplete declarations persist unnoticed. A minority changed what the reused stage computes, and for evaluation stages that means a release decision resting on a measurement of something else.

Three practices follow. First, derive code dependencies automatically. The import closure is cheap to compute and removed most stale reuse while costing about \costCode{} percentage points of executions. Second, treat the environment as a dependency. Requirements and lock files changed in \cmEnv\% of commits, and their effect cannot be ruled out statically; pinning environments by image or lock-file digest makes this dependency explicit. Third, sample reused controls for shadow re-execution, but treat a discrepancy as a signal that the declarations are unsound rather than as a single dependency to patch. The replay shows why: at a \shRho\% sampling rate, patching each exposed dependency removed \shFixTen\% of stale-reuse commits, while switching to automatic declarations after the first exposure removed \shSwitchTen\%. Because most stale reuse is benign, comparing shadow outputs, rather than inputs, also separates the material cases from the harmless ones.

\subsection{Scalability and efficiency}

Reuse avoided most stage executions in every rule except the fail-closed fallback, and the augmented declarations kept that saving while removing most of the risk. The avoided executions are the efficiency gain. The public repositories do not record stage runtimes or energy, so the gain is reported as executions rather than energy. Verification adds tens of microseconds per record and planning remains sub-millisecond for policies of hundreds of controls, so neither limits release frequency. At organizational scale the expensive part is maintaining sound declarations, which is why automatic derivation matters more than faster verification.

\subsection{Running the gate at scale}

The gate is stateless apart from the policy and the keyring, so it can be replicated and candidates can be sharded across replicas by pipeline. The measured throughput of one laptop process exceeds the commit rate of the whole corpus by orders of magnitude, so verification cost is not the obstacle to deployment. Three integration points matter more. First, evidence records map directly onto in-toto attestation statements~\cite{torresarias2019intoto}: the artifact digests are the subject, and the measurement, decision rule, contract digest, and context digest form the predicate. Signing them through Sigstore~\cite{newman2022sigstore} with the executor's workload identity, and logging them in a transparency log, gives the append-only store that the threat model assumes. Second, declared dependencies should not be written by hand. Lineage events such as OpenLineage run facets~\cite{openlineage2026} record the inputs that a job actually read, and combined with the static import closure studied here they can populate a control's contract automatically. Third, on shared clusters the gate fits naturally at admission time, before a release job is scheduled, so that a candidate with stale evidence is held without consuming compute. In HPC and data-center settings the saving that matters is the evaluation work avoided, and the replay shows that most of it can be kept while removing most stale reuse.

\subsection{Threats to validity}

\emph{Construct validity.} The reference dependency set is produced by static analysis and is itself incomplete. It misses dynamic imports, computed file paths, and data outside the repository, so the reported stale-reuse rates are lower bounds. It can also over-approximate, because importing a module does not mean using the part that changed; the audit measures this, and the benign category is large. \emph{Internal validity.} The analyzer emulates DVC's invalidation rule rather than running DVC, and an earlier audit round found two behaviors that needed emulating (template resolution and Windows path separators) before the final run. \emph{External validity.} Public DVC repositories include many course and personal projects~\cite{kalliamvakou2014perils}. Commits are not release candidates, and industrial pipelines may declare dependencies more carefully or less. Results are reported per commit and per repository to show how concentrated the effect is. \emph{Re-execution.} Only stages that run offline from a plain checkout could be re-executed, which favors small projects with data in the repository; the re-execution results describe that subset. \emph{Prototype.} The gate is not integrated with an orchestrator, model registry, or transparency log, and its scenarios test logic under the stated threat model rather than resistance to a compromised signer.

\section{Conclusion}\label{sec:conclusion}

Reusing assurance evidence is safe under conditions that can be stated precisely. A gate can enforce all of them except one, the soundness of dependency declarations, and that one is violated in real ML pipelines often enough to matter. In \nRepos{} public DVC pipelines, declared-dependency reuse avoided \avoidDecl\% of stage executions and reused a stale stage in \unsafeDecl\% of commits. Adding static import closures removed most of that exposure while keeping most of the saving. Evaluation stages, whose outputs are the evidence itself, were the most exposed, and a public-sector valuation pipeline showed the same gap until its maintainers closed it in the way the analysis recommends. Governance platforms that reuse evidence should derive dependencies automatically, treat environments as dependencies, use sampled re-execution to detect unsound declarations rather than to repair them one at a time, and keep the gate independent of the planner. Replaying industrial pipelines with measured stage runtimes, and tracing dependencies dynamically rather than statically, are the most useful next steps.

\section*{CRediT authorship contribution statement}
\textbf{Atanu Mukhopadhyay:} Conceptualization, Methodology, Software, Validation, Formal analysis, Investigation, Data curation, Visualization, Writing -- original draft, Writing -- review \& editing. \textbf{Shravya Reddy Nalla:} Software, Formal analysis, Validation, Writing -- review \& editing.

\section*{Declaration of competing interest}
The authors declare no known competing financial interests or personal relationships that could have appeared to influence the work reported in this paper.

\section*{Funding}
This research did not receive any specific grant from funding agencies in the public, commercial, or not-for-profit sectors.

\section*{Data availability}
The analysis code, the prototype, the list of repositories with the commit identifiers analyzed, the audit labels, and the summary results are provided as ancillary files with this preprint. The complete per-commit and per-stage results are available from the authors. The mined repositories are public; their contents are not redistributed.

\section*{Declaration of generative AI and AI-assisted technologies in the manuscript preparation process}
During the preparation of this work, the authors utilized OpenAI Codex and Anthropic Claude to check English grammar, spelling, and references. After employing these tools, the authors reviewed and edited the content as needed, taking full responsibility for the published article's final content.

\end{document}